\documentclass[runningheads,a4paper]{llncs}

\usepackage[toc,page]{appendix}
\usepackage{amssymb}
\usepackage{graphicx}
\usepackage{amsmath}

\usepackage{amssymb}
\usepackage{graphicx}
\usepackage{caption}
\usepackage{subcaption}
\usepackage{algorithm}
\usepackage{algorithmic}
\usepackage{url}

\urldef{\mailsa}\path| Hossein.Khani8564@gmail.com,|
\urldef{\mailsb}\path| Afsharchim@znu.ac.ir|
\begin{document}

\mainmatter  % start of an individual contribution

\title{Security Games with Ambiguous Beliefs of Agents}

%% a short form should be given in case it is too long for the running head
%%\titlerunning{Lecture Notes in Computer Science: Authors' Instructions}
%
%% the name(s) of the author(s) follow(s) next
%%
%% NB: Chinese authors should write their first names(s) in front of
%% their surnames. This ensures that the names appear correctly in
%% the running heads and the author index.
%%
%\author{Hossein Khani,
%Mohsen Afsharchi
%}
%%
%\authorrunning{Security Games with Ambiguous Beliefs of Agents}
%% (feature abused for this document to repeat the title also on left hand pages)
%
%% the affiliations are given next; don't give your e-mail address
%% unless you accept that it will be published
%\institute{Institute for Advanced Studies in Basic Sciences,\\
%Zanjan University\\
%\mailsc\\}
%\url{http://www.springer.com/lncs}}
\author{Hossein Khani\textsuperscript{1}%
%\thanks{Please note that the LNCS Editorial assumes that all authors have used
%the western naming convention, with given names preceding surnames. This determines
%the structure of the names in the running heads and the author index.}%
\and Mohsen Afsharchi\textsuperscript{2}}

\institute{\textsuperscript{1}Institute for Advanced Studies in Basic Science\\
\mailsa\\
\textsuperscript{2} University of Zanjan, Iran\\
\mailsb}

%\toctitle{Lecture Notes in Computer Science}
%\tocauthor{Authors' Instructions}
\maketitle

\begin{abstract}
Currently the Dempster-Shafer based algorithm and Uniform Random Probability based algorithm are the preferred method of resolving security games, in which defenders are able to identify attackers and only strategy remained ambiguous. However this model is inefficient in situations where resources are limited and both the identity of the attackers and their strategies are ambiguous. The intent of this study is to find a more effective algorithm to guide the defenders in choosing which outside agents with which to cooperate given both ambiguities. We designed an experiment where defenders were compelled to engage with outside agents in order to maximize protection of their targets. We introduced two important notions: the behavior of each agent in target protection and the tolerance threshold in the target protection process. From these, we proposed an algorithm that was applied by each defender to determine the best potential assistant(s) with which to cooperate. Our  results showed that our proposed algorithm is safer than the Dempster-Shafer based algorithm.

\end{abstract}

\section{Introduction}
With the current state of politics, economics, and conflicting ideologies, security has become an ever increasing concern and the driving force behind much strategic development. For example, the security game solving algorithm of DOBSS [1] serves as the core of the ARMOR system, which has been successfully utilized in security patrol schedule at Los Angeles International Airport [2,3].

 In such situation, limited resources pose a constant challenge issue to providing full security coverage.  Examples of  limitations include restricted finances, man power, supplies, etc. In these situations, defenders must often recruit outside agents to aid in the protection of their targets. So, defenders need to do a decision making process to choose some outside agents as cooperator.
 The situation becomes further exacerbated when information on the outside agents is unavailable or ambiguous.
  Ambiguous information refers to the uncertainty of the defender in relation to the objective and behavior of outside agents were they to engage in such cooperation process.

Much literature on the topic of security games attempts to address the issues involved with ambiguous information. In these games, the defenders cannot assign a value for the probability of the outside agent's objective(i.e., whether these agents  will behave more like attackers or defenders in a cooperation process). Here Bayesian models cannot be applied as they require that the defenders be able to assign a precise probability value for the outside agents [4]. For example, consider a battlefield in which security forces try to protect vital resources, while the enemy tries to penetrate their defence and commit espionage. When the security forces employ outside forces to reinforce their strength, they cannot determine for certain whether the i.e., which resources are being targeted. 

In this article, we will refer to the agents who are responsible to protect the target as "defenders" and the outside agents who are asked by the defenders to enhance the protection process are refered to "potential assistants".
Based on the need of defenders to form cooperation processes to enhance protection, and also ambiguity of the decision framework, we propose a model to handle uncertainty in security games. In order to handle such situations, we use the Choquet Expected Utility [5].
Furthermore, we will define two factors of the agents which are used by the others to choose appropriate agents as cooperators. From the agent's point of view, an appropriate agent is  one with whom cooperation  leads to an increase in  payoff  for the defender.

In this article (i) We deal with ambiguous information of the agents about the other agents of different types, (ii) We propose two notions that are based on human behavior, to calculate the ambiguity degree of one agent's belief about the other one, (iii) We propose an algorithm for solving security games with such ambiguous information, (iv) We evaluate our algorithm - and find that our model is efficient and safe for handling security games.

	The rest of the article is organized as follows. Section 2 recaps the Choquet Expected Utility and its usage in decision making. Section 3 introduces our security game as a formal game and proposes an algorithm to solve it. Section 4 describes the experiments conducted to evaluate the performance of the algorithm.  Finally, Section 5 concludes our article.

\section{Background}
This section outlines some backgrounds and describes some key components of decision making under ambiguity.
In this study the cooperation process model is based on that employed by Shehory and Sykara and Sless, Hazon, and Kraus in their respective studies[6,7]. To serve their goals, the agents have to participate in a cooperation process by some amount of their capabilities. In this model, the agents request a percent of the obtained payoff from their cooperation and they gain from the cooperation in accordance to their level of participation and the amount of request they suggest.

Marinacci el al studied details of cooperation of the agents in the presence of ambiguity [5]. The ambiguous setting was described by Paolo Ghirardato [8] as an extension of maximum expected utility which is introduces by Savage[9]. The Ellsberg paradox [10] is an example that shows the importance of studying ambiguous situations and how they differ from certain situations. The Ellsberg Paradox concerns subjective probability theory, which fails to follow the expected utility theory.
To handle ambiguous situations, they defined some real-valued function as neo-additive probabilities, which are called capacities. Using these capacities, they introduced the Choquet Expected Utility as a generalization of expected utility. \\
In the presence of ambiguity, the majority of agents respond by behaving cautiously. Such cautious behavior is referred to as ambiguity-aversion. Bade [11] explains the effect of ambiguity aversion on equilibrium outcomes based on the relaxation of randomized strategy. In contrast, a minority of agents behave  more carelessly, which is referred to as ambiguity-preference. The majority of the time, agents do not act purely in accordance to either ambiguity aversion or ambiguity preference, but rather somewhere in between. 
Also, Wakker [12] extends the notions of ambiguity to arbitrary events and characterizes optimistic and pessimistic attitudes.\\ 

To better understanding the rest of the article, we present below some review of the works done on the field of ambiguity.

We assume the uncertainty a decision maker faces can be described by a non-empty set of states, denoted by $S$. This set maybe finite or infinite. Associated with the set of states is the set of events, taken by sigma-algebra of subsets of $S$, denoted by $\epsilon$.
We assume for each $s$ in $S$, $\{s\}$ is in $\epsilon$ . Capacities are real-valued functions defined on $\epsilon$ that generalize the notion of probability distributions. Formally, a capacity is a normalized monotone set function [13].
\begin{definition}
A capacity is a function $\nu : \epsilon \Rightarrow \mathcal{R}$ which assigns real numbers to events, such that $ E,F \in \epsilon$ and $F \subseteq E$ implies $\nu(E) \geq \nu(F)$. Also $\nu(\phi) = 0$ and $\nu(S) = 1$.
\end{definition}
Let $f:S \rightarrow \mathcal{R}$ be a $\epsilon$-measurable real-valued function. According to the state $s$ the decision maker(agent) decides to execute an action whose results lead to a state from a finite set of states $s^{'}$ and obtaining payoff $f(s^{'})$. That is, the set $f(S)$ is finite.
The Choquet integral can therefore be written in the following intuitive form.\\
\begin{definition}
For any simple function $f$ the Choquet integral with respect to the capacity $\nu$ is defined as:
\begin{eqnarray}
\int f d\nu=\Sigma_{t\in f(s)}[\nu({s|f(s)\geq t})-\nu({s|f(s) > t})].
\end{eqnarray}
\end{definition}
The Choquet integral is interpreted as the expected value of the function $f$ with respect to the capacity $\nu$ [5].
Our research make use of a special kind of capacity called the neo-additive capacity.
Neo-additive capacities can be viewed as a convex combination of an additive capacity and a special capacity that only distinguishes between whether an event is impossible, possible or certain.
To introduce this kind of capacity, suppose the set of events, is partitioned into three subsets: the set of all null events ($\mathcal{N}$), the set of universal events ($\mathcal{U}$) and the set of special events($\epsilon^{*}$).
The null set of events is the set of events that are impossible to occur. Also, the universal set is the set of events which are certain to occur. Finally, every other set is essential in the sense that it is neither impossible nor certain.
%\begin{definition}
%Fixing the set of null events ($\mathcal{N}\subset \epsilon$) and fixing $\alpha$ in $[0,1]$, the Hurwitz capacity, which exactly congruent with $\mathcal{N}$ and with an $\alpha$ degree of optimism, is defined to be:
%
%\begin{equation}
%\mu_{\alpha}^{\mathcal{N}}= \left \{ \begin{array}{ll}
%0 & if~ E\in \mathcal{N}\\
%\alpha & if~E\notin \mathcal{N}~and~ S/E\notin \mathcal{N}\\ 
%1 & if~ S/E \in \mathcal{N}}
%\end{array} \right.
%\end{equation}
%\end{definition} 
The Hurwitz capacity can be viewed as a convex combination of two capacities, one of which reflects complete ignorance or complete ambiguity in everything bar a universal event occurring and the second which reflects complete self-confidence in everything bar null events. Formally, we define a neo-additive capacity as a convex combination of a Hurwitz capacity and a congruent additive capacity.\\
\begin{definition}
For a given set of null events $\mathcal{N}\subset \epsilon$, a finitely additive probability distribution $\pi$ on $(S,\epsilon)$, which is congruent with $\mathcal{N}$ and a pair of numbers $\sigma ,\alpha \in [0,1]$, a neo-additive capacity $ \nu(.|\mathcal{N},\pi , \sigma , \alpha)$ is defined as:\\
\begin{eqnarray}
\nu(E|\mathcal{N},\pi , \sigma , \alpha)=(1-\sigma)\pi(E)+\sigma \mu^{\mathcal{N}}_{\alpha}(E).
\end{eqnarray}
\end{definition} 
For the purpose of this research, each agent shall have a set of possible types $\epsilon$ which can take one of them as its career to participate in a cooperation. It can be interpreted as mathematical definition of various intentions of agent when it starts a cooperation. $E \subset \epsilon$ is equivalent to one specific type that the agent can choose, $\pi$ is the probability distribution that an agent assigns over types of the others initially. Meanwhile, $\sigma$ is the ambiguity degree that the agent has about assigning a probability distribution over types of the others. Finally $\alpha$ is the optimistic or pessimistic attitude toward making decision about types of the others. 
It is straightforward to derive the Choquet integral of a simple function $f$ with respect to a neo-additive capacity:
\begin{definition}
The Choquet expected value of a simple function $f:S\Rightarrow R$ with respect to the neo-additive capacity $ \nu(.|\mathcal{N},\pi,\sigma,\alpha)$ is given by:
\begin{eqnarray}
\int f d\nu =(1-\sigma) E_{\pi}[f]+\sigma(\alpha max\{x:f^{-1}(x) \notin \mathcal{N}\} \nonumber\\ + (1-\alpha) min\{y: f^{-1}(y) \notin \mathcal{N}\})
\end{eqnarray}
\end{definition}
It is important to note that neo-additive capacities satisfy three conditions: $(i)$ They are additive for pairs of disjoint events which are not null and do not form a partition of a universal event.$(ii)$ They exhibit uncertainty aversion for some events. $(iii)$ They exhibit uncertainty preference for some other events.
The ambiguous information of agents represented by neo-additive capacities face some decision problems optimistically and face others’ pessimistically, i.e., there may exist two agents with the same belief and ambiguity degree $(\pi,\sigma)$. The one with smaller $\alpha$ is more ambiguity-averse.

\section{Methods and Materials}

This section explains our security game model and its corresponding details. It presents a method to solve such games.

\subsection*{Problem ِDescription} 
Consider a domain that contains a set of $n$ targets and a set of agents. The set of agents is partitioned into a set of  defenders and a set of  potential assistants. Each defender is responsible for protecting only one target. Based on the amount of energy that the agents use in  target protection, the owner of the target must  reward or punish them accordingly. The amount of reward or punishment is referred to as payoff. Due to the limitation of the defenders on  power supply, they  must rely on outside assistance to enhance their ability to protect. The main issue in this context is that there are diverse types of agents that precise probability cannot be assigned to them by the others.
Thus the goal of the defenders becomes  finding a way to choose the appropriate cooperators among the set of potential assistants, that will result in an increase in the obtained payoff.

\subsection*{Preliminaries}

For this experiment, we assumed that there are three types of agents: the good, the bad, and the worst. These types show the true objectives of the agents and their tendencies in target protection. The good type of the agent is defined as a non-attacker that intends  to fully protect the target. There defender experiences no reduction in payoff. The bad type of  agent is an attacker with whom cooperating leads to modest loss in the payoff of its cooperators. The worst type of the agent is the one with whom cooperating leads to significant loss in payoff. Since the defender is uncertain of the type (i.e., intentions) of its cooperators, it is possible and likely that the real payoff cooperation will differ from the initial amount expected.
 
	When the defender $d$ cooperates with type $T_{a}$ of a potential assistant which uses $w_{a}$ percent of its ability to protect the target $t$, the payoff the defender obtains from such a cooperation is shown by $r_{d} (w,T_{a})$, in which $w=(w_{a},w_{d})$ and $w_{d}$ is the percent of the amount of ability the defender uses in the target protection. This value is assigned by the owner of the target to the defender $d$ according to the specifications of the cooperator (i.e., Its type and the amount of ability used in the protection process). Suppose potential assistant $a$ is chosen by the defender $d$ of type $T_{d}$ as a cooperator. The obtained payoff of potential assistant $a$ is shown by $r_{a} (w,T_{d})$. Since good types of potential assistants employ more ability in the protection process, the defender obtains higher payoff that what it obtained before cooperation. However if the defender $d$ assesses the type $T_{a}$ incorrectly, i.e., assessing the type $T_{a}$ the worst type as a good type, then the real payoff obtained by defender $(r_{d} (w,T_{a}))$ might be lower than what it obtained previously. This is true about the payoff which is obtained by the assistant $a$ when  cooperates with type $T_{d}$ of a defender.

Essentially, the decrease in payoff of the defender $d$ is interpreted as failure of the defender in choosing the best type of potential assistant as a cooperator and the consequent  decrease in payoff depends on the type of potential assistant selected. Conversely an increase in the payoff of the defender $d$ is seen as a success in choosing the appropriate type of assistant. The increase in payoff depends on how potential assistants can be helpful in the protection process(i.e., How much ability they employ in the protection process). This situation is true about the increase and decrease in payoff of potential assistants in a cooperation process.

All of the agent(defenders and potential assistants) use two important notions to assess the other agents before participating any cooperation. The first notion is the "behavior" of the agents in reference to a certain target. This notion is defined as an ordered pair of the amount of ability the  agent uses to protect the target and the requested amount of payoff from the what expected to obtain from cooperation while protecting the target.
Defenders utilize this notion as an index to measure the loyalty of the agent to the target. 
The other important notion is  known as the  "tolerance threshold," which is an index between zero  and one. This notion is applied by each agent in target protection to obtain a value representing the degree of self-confidence the agent has in regards to the amount of ability it uses in the protection process. To better understand the notions, consider the following example:

Suppose the behavior of the agent $i$ about the target $t$ is represented by ($a_{i}^{t},b_{i}^{t}$), also the tolerance threshold of the defender $j$ about the ability it uses in protecting the target $t$ is $TT_{j}^{t}$. The concept behind these two notions and the method in which agents make their selections is based upon human behavior. These notions are selected by the agent subjectively according to its needs. For example the agent may need to increase its payoff by $b_{i}^{t}$, so its decides to protect the target by increasing the amount of its participating ability in the cooperation process, and request its expected payoff. 

Furthermore, if  resources are  restricted, the agent will  exercise more  caution in utilizing  its ability to protect the target(the agent would be aware not to waste its resources carelessly), while the degree of self-confidence or the tolerance threshold  decreases. Note that confidence is inversely related to ambiguity, i.e., more confidence means less ambiguity and vice versa. In our model, the agents can precisely recognize the behavior and the tolerance threshold of  one another. In addition, they  can calculate the real payoffs   given various situations.
The  application of  these notions to evaluate the ambiguity degree is as follows:

If $\frac{b_{i}^{t}}{a_{i}^{t}} \leq TT_{j}^{t} $ holds, the agent $j$ can confidently determine the type of the agent $i$  with which it is interacting in order to protect the target $t$. There is no ambiguity; the agent  is loyal. On the other hand, if $\frac{b_{i}^{t}}{a_{i}^{t}} > TT_{j}^{t} $ holds, the agent $j$ cannot verify the loyalty of the agent $i$, i.e., ambiguity about type of the agent is present. In this scenario, the ambiguity degree of agent $j$ about agent $i$ in cooperation to protect target $t$, $\sigma_{j,i}^{t}$, can be computed as follows:

\begin{eqnarray}
\sigma_{j,i}^{t}=\frac{log_{2}(\frac{b_{i}{t}}{a_{i}^{t}}-TT_{j}^{t})}{log_{2} TT_{j}^{t}}
\end{eqnarray}

We can  verify  the behavior of  this function by exploring the impact of different variables, such as tolerance threshold.\\
Suppose the agent $i$ has a specific behavior $(a_{i}^{t}, b_{i}^{t})$ to deal with the target $t$. As more impatient the agent $j$ is to deal with the target $t$, the lower its tolerance threshold becomes and the numerator of the fraction  increases. Moreover a decrease in the tolerance threshold results in a decrease in the denominator. Meanwhile, increase in the numenator and decrease in the denumenator lead to increase in the ambiguity degree of agent $j$ about agent $i$, $(\sigma_{j,i}^{t})$.
Inversely, the more patient the agent $j$ is, the higher it's tolerance threshold. In this case the numerator of the function will be lower and the denominator higher, which translates to a decrease in the ambiguity degree.

As shown  earlier, the ambiguity degree of the agents towards the types of potential assistants is computed using notions of tolerance threshold and behavior. There we have defined our game as an ambiguous security game in which each agent has some ambiguity degree in respect to its potential assistants in different cooperations.

\subsection*{Formal Definition of the Game}
Using the information obtained from the game, we define our game as $(N, T, B, S, r, U, PI) $ such that:\\

\begin{itemize}
	\item $N=(d,pa)$ is the set of players, where $d$ stands for defender and $pa$ stands for potential assistants.
	\item $T$ contains elements of $T_{i}$ which shows the type set of agent $i$.
	\item $B=\{B_{i}| i\in N\}$ represents the behavior set of each agent.
	\item $S=S_{d} \times S_{p_{a}}$ where $S_{d}=S_{p_{a}}=\{s_{1},s_{2},...,s_{n}\}$ is the pure strategy set of the attacker and the defender representing the different values for the tolerance threshold of each agent.
	\item $r=\{r_{i}(w,T{j})| i,j \in N\}$ and $r_{i}(w,T{j})$ is the real payoff obtained by player $i$ when it cooperate with player $j$. $w$ is the ordered pair of the amount of ability used by the player $j$ and the amount of ability used by the player $i$. 
		\item $U=\{U_{i}(s_{i},(a_{j},b_{j})) | i,j \in N ~\&~s_{i}\in S_{i}~ \&~ (a_{j},b_{j})\in B\}$ is the payoff that the player $i$ expects to obtain according to strategy profile $S$ while it cooperates with $j$ whose behavior is $(a_{j},b_{j})$.
	\item $PI$ is the set of initial beliefs $\pi$ of players about type of each other.
\end{itemize}

The goal of the defender is to examine different possible values for its tolerance threshold and select the one which  maximizes  payoff.  In our security game, first the defender commits an optimal strategy (selects its optimal tolerance threshold) to the potential assistants, and then the potential assistants try to find the optimal strategy for themselves [14]. Because each defender is responsible for protecting only one target, we suppose that its expected payoff is obtained through all available potential assistants in the cooperation. It is also worth mentioning that the expected payoff of the potential assistant is  to be received from the defender regardless of the attendance of the other potential assistants in the cooperation.  Given the above definition, the payoff that the defender $d$ expects to obtain with tolerance threshold $s_{d}$ when taking agent $p$ with behavior ($a_{p},b_{p}$) as its assistant to protect the target $t$  $(u_{d}(s_{d},(a_{p},b_{p})))$ is expressed by the following:
\begin{eqnarray}
u_{d}(s_{d},(a_{p},b_{p}))=(1-\sigma_{d,p}^{t})\pi(t_{p})+\sigma_{d,p}^{t}[\alpha max_{t_{p}} r_{d}(a_{p},t_{p}) \nonumber \\+(1-\alpha)min_{t_{p}} r_{d}(a_{p},t_{p})].
\end{eqnarray}
This formula is a variant of formula (4).
In this formula, $\sigma_{d,p}^{t}$ is the ambiguity degree of the defender in relation to the potential assistant $p$, this value can be computed by formula (5). $\pi(t_{p})$ is the initial belief of the defender on the types of the opponents, which $t_{p}$ shows a specific type from all available types of the potential assistant. $\alpha$ is the degree of optimism, which shows how optimistic the defender is in its computations. 
Suppose the set of potential assistants who cooperate with the defender to protect the target $t$ is shown by $C_{t}$, then the total payoff which is obtained by defender $d$ to protect target $t$ is represented by:

\begin{align}
U_{d}((s_{d},s_{p}),B)=b_{d}^{t}\times \Sigma_{i \in C_{t}} u_{d}(s_{d},b_{i})
\end{align}
\\
The payoff of the potential assistant, as it begins to assist the defender, is computed by:

\begin{eqnarray}
U_{p}(s_{p},(a_{d},b_{d}))=b_{p}((1-\sigma_{p,d}^{t})\pi(t_{p})+\sigma_{p,d}^{t}[\alpha max_{t_{d}} r_{p}(a_{d},t_{d}) +\nonumber \\ (1-\alpha)min_{t_{d}} r_{p}(a_{d},t_{d})]).
\end{eqnarray}\\
This formula is a variant of formula (4). Note to the difference between the number of arguments of function $U_{d}$ and $U_{p}$. It is important to know that in formula (6) and (8), if $\sigma_{p,d}^{t}=0$, these formulas are reduced to Savage Expected Utility. If $\alpha=0$ the agent is on pessimistic view and if  $\alpha=1$, it is in its optimistic view.
Given the defenders strategy $s_{d}$ the optimal strategy of the potential assistant is $s_{a}^{*}$ if:

\begin{eqnarray}
U_{p}(s_{a}^{*},(a_{d},b_{d})) \geq U_{p}(s_{a}^{'},(a_{d},b_{d}))~~ \forall s_{a}^{'} \in S_{a}
\end{eqnarray}

In the ambiguous security game, given the defender strategy $s_{d}$, if the optimal response strategies of the potential assistant are $s_{a}^{*}$, then $s_{d}^{*}$ is the defender’s optimal strategy if:

\begin{eqnarray}
U_{d}(s^{*}_{d},s_{a}^{*},B) \geq U_{d}(s^{'}_{d},s_{a}^{*},B)~~ \forall s_{d}^{'} \in S_{d}
\end{eqnarray}

\subsection*{Algorithm}
As mentioned previously, the defender executes its optimal strategy first. To accomplish this, it must  evaluate the strategies of the potential assistants and choose the best one  based on the strategies which are chosen by the potential assistants. 

The defender finds its optimal strategy in two phases. In the first phase, he needs to recognize which one of the potential assistants is willing to cooperate with him in the cooperation process. The potential assistant is willing to cooperate with defender $d$, if the payoff the defender expects to obtain from this cooperation is higher than its previously obtained payoff. In the second phase, the defender $d$ starts to compare its different strategies on the basis of the amount of payoff it expects to get.

In the first phase, the defender calculates the ambiguity degree for each potential assistant and all possible strategies they can choose using formula (5). According to the different values obtained, the defender  uses formula(8) to compute the expected payoff for each potential assistant. The defender then compares, for each potential assistant, the different strategies that they can choose on the basis of their corresponding expected payoff and determines their optimal strategy.  The defender finds, for each potential assistant, their optimal strategy and their corresponding expected payoff (formula 9). If the maximum expected payoff of a potential assistant is higher than its previously obtained payoff, the defender will know that the potential assistant is willing to cooperate with him. The first phase concludes with the defender having identified the set of potential assistants who are willing to cooperate. This set is shown by $C_{t}$ where $t$ is the target that the defender is responsible to protect.

In the second phase, the defender evaluates its options through a similar process to the first phase. Using formula (5), it quantifies ambiguity degrees and from formula (6) the corresponding expected payoff for each strategy. The winning strategy of the defender is the one that maximizes its expected payoff (formula 10). The algorithmic process of the game solution is shown in algorithm 1. 

\begin{algorithm}
\caption{Choquet Expected Utility Based Solution Algorithm}
\begin{algorithmic}
\REQUIRE $ Strategy~set (S), Type~set(T),$ \\ $Behaviuor~set(B), reward~and~penalties$
\ENSURE $S^{*}$
\FOR {Each strategy from defender's strategy set}
\FOR {Each potential assistant}
\STATE Use formula (8) and (9) to find the optimal strategy of the potential assistant and his expected payoff in the view of the defender. 
\ENDFOR
\STATE Compare the expected payoff of potential assistants with their previous payoff and find the set of agents who are willing to cooperate with defender $(C_{t})$
\ENDFOR

\STATE find the maximum payoff of the defender obtained while he cooperates with the set of willing potential assistants using  formulas (6) and (7), $s^{*}=argmax_{s} \Sigma_{i \in C_{t}} U_{d}(S,B)$

\end{algorithmic}
\end{algorithm}
The time complexity of the algorithm is computed in the theorem below.\\
\begin{theorem}
In security games with ambiguous agent types, the complexity of finding optimal pure strategies of a selected defender, is $O(\#players,\#outcomes)$.
\end{theorem}

\begin{proof}
Each pure strategy to which the first player may commit will induce a subgame for the remaining players. We can solve each such subgame recursively to find all of its optimal strategy profiles; each of these will give the original first player some utility. Those that give the first player maximum utility correspond exactly to the optimal strategy profiles of the original game.
\end{proof}

Note that if the defender $d$ assesses the type of the potential assistant ($a$) as a cooperator wrongly, it means $T_{a}$ is the worst type but the defender miscategorizes it as a good type, the real obtained payoff by the defender would decrease.
The defender $d$ inaccurately assesses the type of the others when it  miscalculates the expected payoff . An error in  computation is  a result of flawed  initial beliefs apropos the other agents
\section{Experiments}
In this section, we are going to evaluate our model via experimentation. One of the more realistic works is done by Zhang and his colleagues in[14].
They use a model (D-S theory based model) to handle the ambiguous information about the types of the attacker in a security game. They define a basic probability assignment or a mass function and use it to define a preference degree of each agent over different strategies. 
In order to prove the efficiency of our proposed algorithm we are going to evaluate our model by the D-S theory based solution concept.
To use  D-S theory,  the mass function that the agents assign to different events  must include  capacities  used in our algorithm. In order to incorporate  capacity from our algorithm to the mass function, we used the following method [15]: 

\begin{eqnarray}
m(S)=\Sigma_{H\subset S} (-1)^{(|S|-|H|)} \nu(H)
\end{eqnarray}

Now we use the experiments to show that, the worst case in our model is better than the method which uses D-S theory based model.

To perform our evaluations, we suppose that there are five targets and a set of agents who want to protect the targets and gain payoff from their owners. At first, every target is protected by only one agent as defender. The remaining agents are called potential assistants which are going to be used when necessary. For simplicity, we assume that at first, agents who work as defender gain a fixed amount of payoff from owners of the targets and potential assistants gain zero payoff. Also, we suppose that all of our agents can be in one of three types which are the good type, the bad type and the worst type. 

For simplicity of programming, we initialize each argument of the ordered pair of the behavior of agents by random values in the interval [0,1]. These values show the percentage of expected payoff the agent requested and the percent of the employed ability to protect the target respectively.

In addition, each agent has a tolerance threshold about their amount of employed ability to protect the target. The value which is assigned to the tolerance threshold of the agent is relevant to the amount of its employed ability in the target protection. The higher the amount of employed ability of the agent, the lower the tolerance threshold of the agent will be in the target protection. For example in our algorithm if $0.6<a_{i}^{t}<0.8$ holds, the agent assigns value $ 0.2$ to its tolerance threshold in the cooperation process. If $0.4<a_{i}^{t}<0.6$ holds the assigned value to the tolerance threshold $(TT_{i}^{t})$ will be $0.4$ and so on.

Due to the randomly assigned values to the behavior of the agents, we run our experiments 100 times according to the mentioned algorithm in previous section and use the average of the experiments as the outcome of the algorithm. 

First of all, we evaluate the impact of changing the number of agents on the outcome of our proposed algorithm(CEU-based solution algorithm) and D-S theory based algorithm. We observe that in the worst case, the number of true detections in our algorithm is more than the number of true detections in D-S theory based algorithm. In one of the comparison we assume all the agents are ambiguity-preference( the degree of optimism of all agent is high) and ih the other one we assume the agents are neither ambiguity averse nor ambiguity preference(the degree of optimism of all agent is moderated). It shows that being ambiguity averse or ambiguity preference has no impact of the goodness of CEU-based algorithm (see Figure 1(a), (b)).

It is important to know that in our implementation, the true detections of the defender $d$ are based on comparing the amount of payoff the defender expects to obtain before participating in the cooperation and what it really obtains after participating. If its expected payoff is equal to the real payoff it will obtain after cooperating with a potential assistant, the defender truly detects the type of its assistant.

Another evaluation which is done in our experiments is exploring the impact of changing the number of strategies agents can choose (changing the values can be assigned to tolerance threshold) on the outcome of two algorithms. The result of this evaluation shows that changing the number of strategies has no impact on the goodness of CEU based algorihm to do more true detections (See Figure 2(a),(b)).

Also, we use two metrics to shows how the CEU-based algorithm is better than the D-S theory based algorithm. One of the metrics is "sensitivity" which we use it to compare the two algorithms.(See Figure(3)). Regardless of different number of strategies and different number of types the agent can have, the figure shows that the CEU based algorithm is more sensitive than D-S theory based algorithm.
 We find the sensitivity measure using the formula:
\begin{eqnarray}
Sensitivity=\dfrac{\# of~True~Positives}{\#of~True~Positives + \# of~False~Negatives}
\end{eqnarray}
~~\\
To explore the errors in the two algorithms, we use MRSE measure(Mean Root Square Error). We find the value of normalized MRSE for each algorithm and refer to the difference between them as distance of the worst penalties(See Figure (4)). 
\begin{figure}
\centering
\begin{subfigure}{.5\textwidth}
  \centering
  \includegraphics[width=1.0\linewidth]{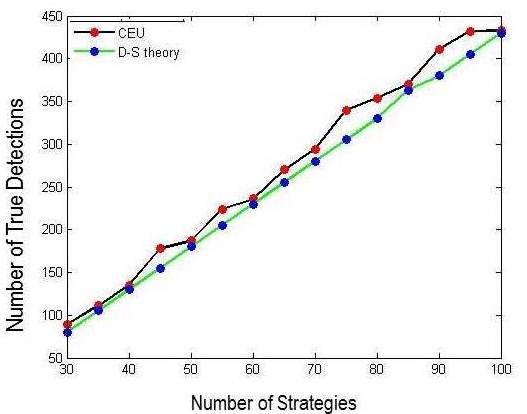}
  \caption{(a) alpha=0.5}
  \label{alpha=0.5}
\end{subfigure}%
\begin{subfigure}{.5\textwidth}
  \centering
  \includegraphics[width=1.0\linewidth]{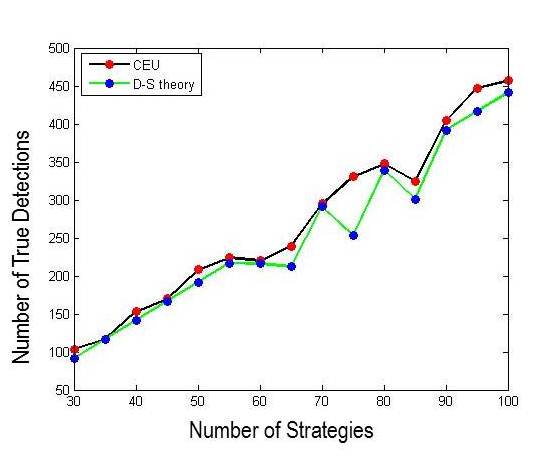}
  \caption{(b)alpha=0.7}
  \label{fig:sub2}
\end{subfigure}
%\begin{subfigure}{.5\textwidth}
%  \centering
%  \includegraphics[width=1.0\linewidth]{alpha8.jpg}
%  \caption{(c)alpha=0.8}
%  \label{fig:sub2}
%\end{subfigure}
\caption{ Comparing the performance of CEU-based solution algorithm with D-S theory based solution agorithm.(number of types=3, number of strategies=8)}
\label{fig:test}
\end{figure}

\begin{figure}
\centering
\begin{subfigure}{.5\textwidth}
  \centering
  \includegraphics[width=1.0\linewidth]{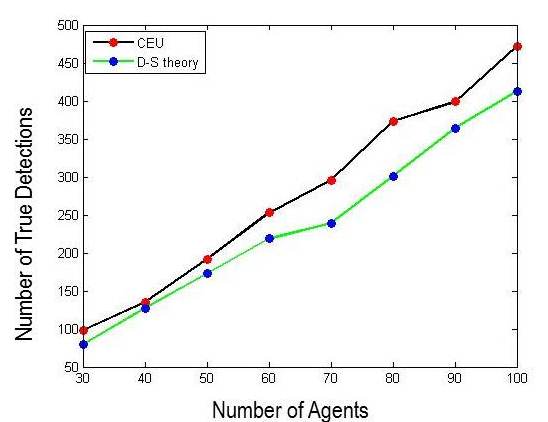}
  \caption{(a)number of strategies=4}
  \label{number of strategies=4}
\end{subfigure}%
\begin{subfigure}{.5\textwidth}
  \centering
  \includegraphics[width=1.0\linewidth]{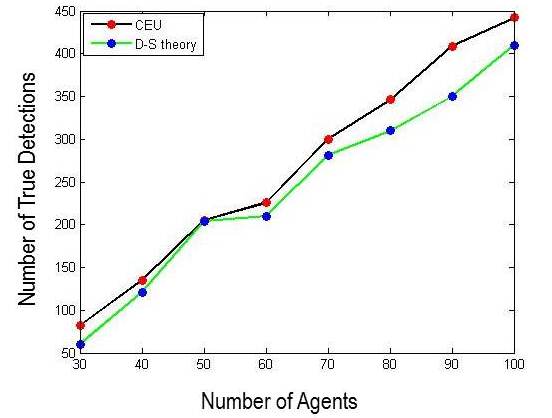}
  \caption{(b)number of strategies=8}
  \label{fig:sub2}
\end{subfigure}
\caption{ Comparing the performance of CEU-based solution algorithm with D-S theory based solution agorithm considering different number of types and strategies each agent could have.(number of types=4)}
\label{fig:test}
\end{figure}

\begin{figure}
\centering
\begin{subfigure}{.5\textwidth}
  \centering
  \includegraphics[width=1.0\linewidth]{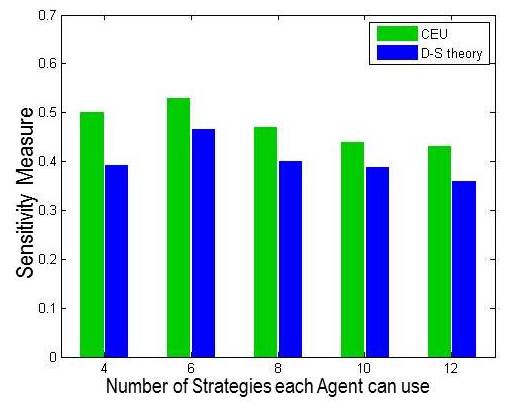}
  \caption{(a)number of strategies=4}
  \label{number of types=3}
\end{subfigure}%
\begin{subfigure}{.5\textwidth}
  \centering
  \includegraphics[width=1.0\linewidth]{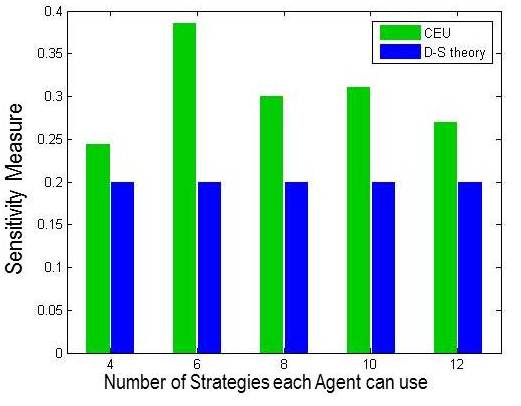}
  \caption{(b)number of types=5}
  \label{fig:sub2}
\end{subfigure}
\caption{ Comparing the Sensitivity of two algorithms considering different number of types for each agent(number of agents=40).}
\label{fig:test}
\end{figure}

\begin{figure}
\centering
\begin{subfigure}{.5\textwidth}
  \centering
  \includegraphics[width=1.0\linewidth]{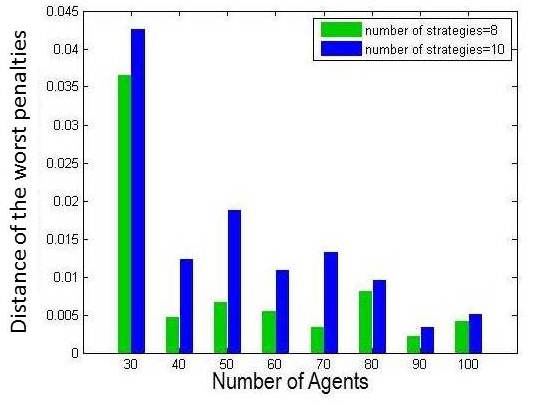}
  \caption{(a) Comparison by different number of strategies}
  \label{number of types=3}
\end{subfigure}%
\begin{subfigure}{.5\textwidth}
  \centering
  \includegraphics[width=1.0\linewidth]{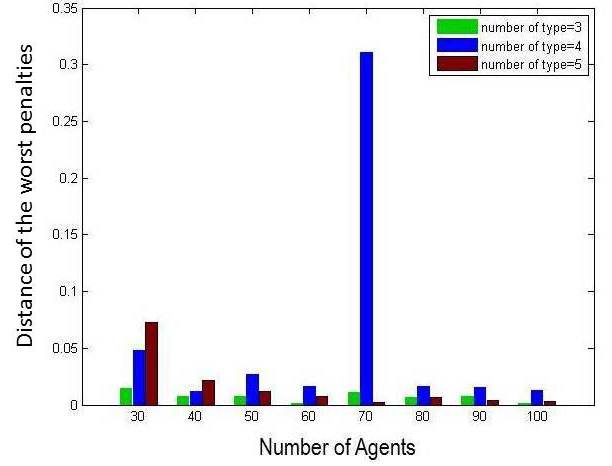}
  \caption{(b) Comparison by different number of types}
  \label{fig:sub2}
\end{subfigure}
\caption{ Comparison of the two algorithms based on distance of the worst penalties.}
\label{fig:test}
\end{figure}

~~\\
Observation1: Our model guarantees more safety than the model based on D-S theory by considering different degrees of optimism.

~~\\
Observation 2: The Choquet Expected Utility solution based algorithm is safer than the model based on D-S theory, by considering different number of types and strategies for each agent.
~~\\
Some evaluations have been done in [14] in order to prove the efficiency of D-S theory based algorithm against the algorithm based on Uniform Random Probability. Based on these experiments, we can conclude that the Choguet Expected Utility based algorithm is safer that the algorithm based on Uniform Random Probability algorithm.
~~\\
\section{Conclusion and Future work}
This paper proposes a new paradigm of security games, in which the defenders on one hand they need the help of potential assistants in order to have incessant protection of their vital targets against attackers due to their limited capabilities. On the other hand, Regard to the presence of ambiguous information of the agents about their tendencies, the defenders cannot decide precisely about selection of their cooperators. We consider that at first the defender execute his optimal strategies and then potential assistants do their tasks, since that the strategies of the defenders are known to potential assistants, so the defender must first consider the strategies that can be chosen by potential assistants to have the better understanding of the situation. We develop an algorithm based on the Choquet Expected Utility to find the defender's optimal strategy and discuss their computing complexity.
Furthermore, we evaluate our model by lots of experiments, we find: (i) The agents face their ambiguity by defining the two notions, behavior and tolerance threshold. (ii) Our model guarantee more safety than the model based on D-S theory by regardless of agents being ambiguity-averse or ambiguity-preference. (iii) We show that our model is more sensitive than D-S theory based model and as a result it is more efficient then Uniform Random probability based algorithm.(iv) We use the Mean Root Square Error (RMSE) to show that errors in CEU-based algorithm are less  than D-s theory based algorithm. So our model can well handle the interaction of agents to protect a common target in the presence of ambiguous information.  

As the future works we are going to deal with equilibria issues in our proposed framework according to the work done by Weber [16] that they investigated the difference among equilibria with respect to various attitudes toward ambiguity, and showed that different types of contingent ambiguity  affect equilibrium behavior, and based on what Kilka [17], and Wu and Gonzalez [18] experimented about this topic. 

\section*{References}
[1] Paruchuri, P., Pearce, J.P., Marecki, J., Tambe, M., Ordonez, F., Kraus, S.: Playing games for security: An efficient exact algorithm for solving Bayesian Stackelberg games. Proceedings of the 7th International Joint Conference on Autonomous
Agents and Multiagent Systems, vol. 2, pp. 895–902 (2008).\\

[2] Tambe, M.: security and game theory: Algorithms, Deployed Systems, Lessons Learned. Cambridge University Press, Cambridge, (2011).\\

%[3] Pita, J., Jain, M., Marecki, J., Ord´o˜nez, F., Portway, C., Tambe, M., Western, C., Paruchuri, P., Kraus, S.: Deployed ARMOR protection: The application of a game theoretic model for security at the Los Angeles International Airport. Proceedings of the 7th International Joint Conference on Autonomous Agents and Multiagent Systems: Industrial Track, pp. 125–132 (2008).

[3] Zimmermann, E.: Globalization and terrorism. European Journal of Political Economy 27(suppl. 1), S152–S161 (2011).\\

[4] Marinacci,M.:Ambiguous Games. Games and Economic Behavior, (31:2), 191-219 (2000).\\

[5] Marinacci,M., Montrucchio,L.: Introduction to the Mathematics of Ambiguity .Dipartimento di Statistica e Matematica Applicata and ICER Università di Torino, (2003).\\

[6] Shehory,O., Sykara,S.:Multi-agent Coordination through Coalition Formation, Lecture Notes in Artificial Intelligence no. 1365, Intelligent Agents IV, A. Rao, M. Singh and M. Wooldridge (Eds.), pages 143-154. Springer, (1997). \\

[7] Sless,L., Hazon,N., Kraus,S., Wooldridge,M.: Forming coalitions and facilitating relationships for completing tasks in social networks, AAMAS, (2014).\\

[8] Ghirardato,P.: Ambiguity. Dipartimento di Matematica Applicata and Collegio Carlo Alberto, Universitate di Torino, (2010).\\

[9] Leonard J. Savage.:The Foundations of Statistics. Wiley, New York, (1954).\\

[10] Segal,U.: The Ellsberg Paradox and Risk Aversion:An Anticipated Utility Approach. University of Toronto, Department of Economics, (1985).\\

[11]  Bade, S.: Ambiguous act equilibria. Games and Economic Behavior, 71(2):246–260, (2011).\\

[12] Wakker,P.: Testing and characterizing properties of nonadditive measures through violations of the sure thing principle. Econometrica, 69:10391060, (2001).\\

[13] Chateauneuf, A., Eichberger, J., Grant, S.:Choice under uncertainty with the best and worst in Mind: Neo-additive capacities. Journal of Economic Theory, 137, 538 - 567(2007).\\

[14] Zhang, Y., Luo X., Ma, M.: Security Games with Ambiguous Information about Attacker Types. Australasian Conference on Artificial Intelligence, volume 8272 of Lecture Notes in Computer Science, page 14-25. Springer, (2013).\\

%\bibitem{5} Pita J., Jain M., Ord´o˜nez F., Portway C., Tambe M., Western C., Paruchuri P., and Kraus S.: Using game theory for Los Angeles airport security. AI Magazine, 30(1):43–57, (2009).

[15] De Marco, G., Maria R.: Beliefs correspondences and equilibria in ambiguous games. International Journal of Intelligent Systems, 27(2):86–102, (2012).\\

[16] Abdellaoui, M., Vossmann F., Weber, M.: Choice-based elicitation and decomposition of decision weights for gains and losses under uncertainty. Management Science, 51:13841399, (2005).

[17] Kilka, M., Weber, M.: What determines the shape of the probability weighting function under uncertainty? Management Science, 47:17121726, (2001).

[18] Wu, G., Gonzalez, R.: Nonlinear decision weights in choice under uncertainty. Management Science, 45:7485, (1999).

%
%\bibitem{4} Tambe, M.: Security and game theory: Algorithms, deployed systems, lessons learned. Cambridge University Press, New York (2011)

%\bibitem{url} National Center for Biotechnology Information, \url{http://www.ncbi.nlm.nih.gov}

%
%\bibitem{CTOPRE}De Marco, G., Maria R.: Beliefs correspondences and equilibria in ambiguous
%games. International Journal of Intelligent Systems, 27(2):86102, (2012).

%\end{thebibliography}
\end{document}